\relax
\documentclass[letterpaper]{article} 
\usepackage{aaai22}  
\usepackage{times}  
\usepackage{helvet}  
\usepackage{courier}  
\usepackage[hyphens]{url}  
\usepackage{graphicx} 
\usepackage{natbib}  
\usepackage{caption} 
\DeclareCaptionStyle{ruled}{labelfont=normalfont,labelsep=colon,strut=off} 
\usepackage{amsfonts} 
\usepackage{eucal} 
\usepackage{amsthm} 
\usepackage{amsmath} 
\usepackage{subfigure} 
\usepackage{multirow} 
\newtheorem{theorem}{Theorem} 
\newtheorem{problem}{Problem} 
\usepackage{bm} 
\theoremstyle{remark}
\newtheorem{remark}[theorem]{Remark}
\usepackage{amsmath}
\usepackage{amssymb}
\usepackage{mathtools}
\usepackage{amsthm}
\usepackage{algorithm}
\usepackage{algorithmic}

\newcommand{\hide}[1]{}
\newcommand*{\Scale}[2][4]{\scalebox{#1}{$#2$}}

\usepackage{newfloat}
\usepackage{listings}
\floatstyle{ruled}
\newfloat{listing}{tb}{lst}{}
\floatname{listing}{Listing}
\title{Unified Group Fairness on Federated Learning}
\author{

    Fengda Zhang\textsuperscript{\rm 1}, Kun Kuang\textsuperscript{\rm 1}, Yuxuan Liu\textsuperscript{\rm 2}, Long Chen\textsuperscript{\rm 1}, Chao Wu\textsuperscript{\rm 2}, Fei Wu\textsuperscript{\rm 1},\\ Jiaxun Lu\textsuperscript{\rm 3}, Yunfeng Shao\textsuperscript{\rm 3}, Jun Xiao\textsuperscript{\rm 1}
    \\
}
\affiliations{

    \textsuperscript{\rm 1}College of Computer Science and Technology, Zhejiang University\\
    \textsuperscript{\rm 2}School of Public Affairs, Zhejiang University\\
    \textsuperscript{\rm 3}Huawei Noah’s Ark Lab

    \{fdzhang, kunkuang, yuxuanliu, chao.wu, wufei\}@zju.edu.cn, {zjuchenlong}@gmail.com,\\
    \{lujiaxun, shaoyunfeng\}@huawei.com,
    junx@cs.zju.edu.cn

}

\usepackage{bibentry}

\begin{document}

\maketitle

\begin{abstract}
Federated learning (FL) has emerged as an important machine learning paradigm where a global model is trained based on the private data from distributed clients. However, most of existing FL algorithms cannot guarantee the performance fairness towards different groups because of data distribution shift over groups. In this paper, we formulate the problem of unified group fairness on FL, where the groups can be formed by clients (including existing clients and newly added clients) and sensitive attribute(s). To solve this problem, we first propose a general fair federated framework. Then we construct a unified group fairness risk from the view of federated uncertainty set with theoretical analyses to guarantee unified group fairness on FL. We also develop an efficient federated optimization algorithm named Federated Mirror Descent Ascent with Momentum Acceleration (FMDA-M) with convergence guarantee. We validate the advantages of the FMDA-M algorithm with various kinds of distribution shift settings in experiments, and the results show that FMDA-M algorithm outperforms the existing fair FL algorithms on unified group fairness.

\end{abstract}

\section{Introduction}
Federated learning (FL) has emerged as an important machine learning paradigm where distributed clients (e.g., a large number of mobile devices or several organizations) collaboratively train a shared global model while keeping private data on clients~\cite{mcmahan2017communication}. 
However, FL may suffer from fairness problem by disproportionately advantaging or disadvantaging the model performance on different subpopulations, which becomes an increasing concern, especially in some high-stakes scenarios such as loan approvals, healthcare, etc~\cite{kairouz2019advances}. 
How to develop a fair FL framework is of paramount importance for both academic research and real applications, and has become an important research theme in recent years~\cite{mohri2019agnostic, li2019fair, wang2021federated}.

\begin{figure}[t]
\centering                         
\includegraphics[width=3.1in]{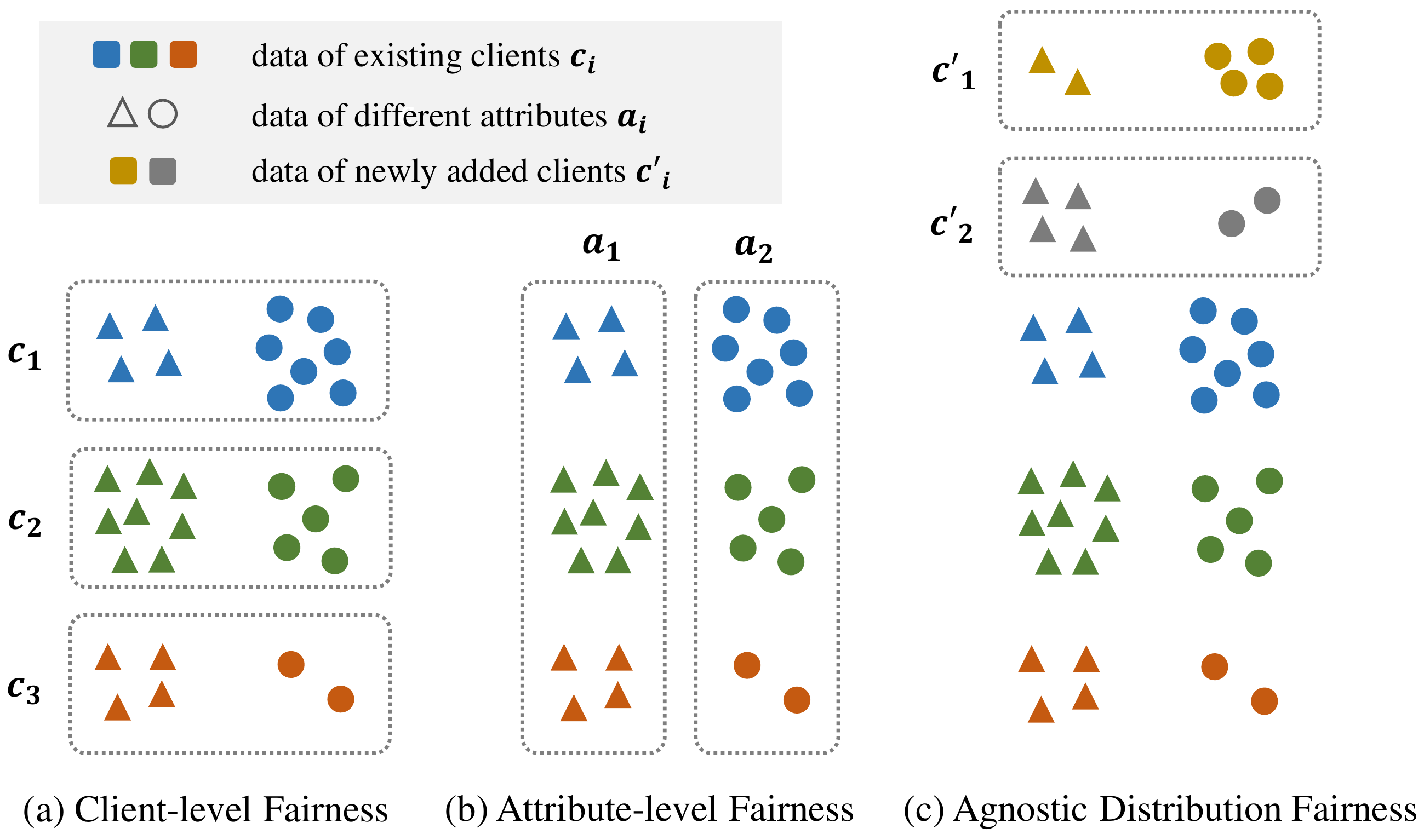}
\vspace{-0.5em}
\caption{Illustration of three kinds of group fairness in FL scenario. (a) client-level fairness: $P(Y=\hat{Y}|c_i)=P(Y=\hat{Y}|c_j)$ for $\forall$ $i$, $j$; (b) attribute-level fairness: $P(Y=\hat{Y}|a_i)=P(Y=\hat{Y}|a_j)$ for $\forall$ $i$, $j$; (c) agnostic distribution fairness: $P(Y=\hat{Y}|c'_i)=P(Y=\hat{Y}|c'_j)$ for $\forall$ $i$, $j$. The groups are formed by existing clients index, sensitive attribute(s), and index of newly added clients with unknown distribution, respectively.}
\vspace{-1.7em}
\label{fig-intro}                       
\end{figure}

A straightforward idea is to apply the methods designed for centralized fairness problem to FL setting. Unfortunately, most of them assume a centralized available training dataset, which infringes the data privacy in FL~\cite{barocas2016big, woodworth2017learning, mehrabi2021survey}. Recently, some works have been proposed to encourage the federated model to have similar performance over different clients (\textbf{\emph{client-level fairness}}, shown in Fig.~\ref{fig-intro}(a)) by learning personalized models for each client~\cite{hanzely2020federated, deng2020adaptive}, reweighting~\cite{mohri2019agnostic, li2019fair} or mitigating conflicts of gradients among clients~\cite{wang2021federated}. However, a federated model trained by them may still suffer from ethical issues in real applications due to neglect of fairness at other levels. For example, consider a scenario that several banks (clients) participate in FL to collaboratively train a loan approval model. Different banks will have different customer demographic compositions (formed by sensitive attribute(s), such as race and/or gender). Although a federated model trained by client-level method treats different banks fairly, the accuracy of it may vary greatly among some sensitive attribute(s) (\textbf{\emph{attribute-level fairness}}, Fig.~\ref{fig-intro}(b)). Moreover, the fairness of federated model on those banks that newly participate in FL (\textbf{\emph{agnostic distribution fairness}}, Fig.~\ref{fig-intro}(c)) is also not guaranteed. A model violating any of the above fairness may lead to serious ethical problems, which naturally leads us to ask: \textit{Can we propose a unified FL framework to train a federated model achieving client-level, attribute-level and agnostic distribution fairness \textbf{simultaneously}?}

In this paper, we focus on the problem of unified group fairness on FL, where the group can be formed by clients (including existing clients and newly added clients) and sensitive attribute(s). To address this problem, we propose a general fair federated framework and a unified group fairness risk from the view of federated uncertainty set~\cite{delage2010distributionally, wiesemann2014distributionally, duchi2016variance}. Theoretically, we prove that the proposed unified group fairness risk provides an upper bound for both client-level and attribute-level fairness risks, which helps to deal with complex distribution shifts and thus guarantee fairness at multiple levels simultaneously.
We also develop an efficient federated optimization algorithm named Federated Mirror Descent Ascent with Momentum Acceleration (FMDA-M) to optimize the proposed risk for unified group fairness on FL, and provide corresponding convergence guarantee.
Empirically, the advantages of our proposed FMDA-M method, in terms of unified group fairness and robustness which refers to the worst performance in groups, are demonstrated under different kinds of distribution shift on three real-world datasets.

Our main contributions can be summarized as follows: (i) We investigate a new problem of unified group fairness on FL for achieving client-level, attribute-level and agnostic distribution fairness simultaneously; (ii) From a novel perspective of uncertain sets in FL, theoretically, we propose a fair federated framework and a corresponding unified group fairness risk; (iii) To optimize the proposed risk for unified group fairness, we develop an efficient federated optimization algorithm named FMDA-M based on mirror descent ascent technique and momentum acceleration with convergence guarantee; (iv) Extensive experiments show that our proposed FMDA-M outperforms the existing fair federated learning methods on unified group fairness.

\section{Related Work}
\textbf{Fairness in Machine Learning.}
Fairness in ML has attracted much attention, which can be divided into two branches: \textit{individual fairness} and \textit{group fairness}~\cite{zemel2013learning, awasthi2020beyond, binns2020apparent}.
Individual fairness encourages the models to treat similar individuals similarly~\cite{biega2018equity, sharifi2019average, mukherjee2020two}, while group fairness requires that the model treats different groups equally~\cite{dwork2012fairness, dwork2018decoupled}.
Here we mainly focus the latter one, which is typically defined via some protected attribute(s) and a metric such as statistical parity, equalized odds~\cite{hardt2016equality}, and predictive parity~\cite{chouldechova2017fair}.
In this paper, we consider how to learn a fair global model that achieves a uniform performance across groups in FL. 

\textbf{Federated Learning and Fairness.}
FL has received much attention as an important distributed learning paradigm~\cite{mcmahan2017communication}. The scope of federated learning studies is broad, which includes statistical challenges~\cite{zhao2018federated, yu2020federated}, privacy protection~\cite{truex2019hybrid}, systematic challenges~\cite{konevcny2016federated1, konevcny2016federated2, suresh2017distributed}, fairness~\cite{mohri2019agnostic, li2019fair}, etc~\cite{kairouz2019advances, yang2019federated}. The existing studies of fairness on FL can be divided into three categories: performance fairness across clients~\cite{li2019fair, mohri2019agnostic, deng2021distributionally, li2021ditto}, model fairness defined on sensitive attributes~\cite{du2021fairness} and incentive mechanism~\cite{kang2019incentive, zhan2020learning}.
In this paper, we focus on the performance fairness across clients (including existing clients and newly added clients) and sensitive attribute(s).


\textbf{Distributionally Robust Optimization.}
There has been a surge of interest in distributionally robust optimization (DRO), which can deal with distribution shifts by considering a potential distribution set around the original distribution and optimizing the worst case~\cite{delage2010distributionally, wiesemann2014distributionally}. There are mainly two definitions of distance between distributions in DRO: \textit{$f$-divergences} and \textit{Wasserstein distance}. 
The former method is effective when the support of
the distribution is fixed~\cite{duchi2016variance, namkoong2016stochastic, duchi2021learning}, while Wasserstein distance-based DRO considers the potential distributions with different supports and allows robustness to unseen data, but is difficult to optimize~\cite{sinha2017certifying, esfahani2018data, liu2021stable}.
Recently, some studies about group DRO have emerged~\cite{hu2018does, sagawa2019distributionally, oren2019distributionally}, which considers the distribution shifts over groups. In this paper, we extend DRO to FL setting for unified group fairness.


\section{Problem Formulation}

\subsection{Preliminary on Federated Learning}
Suppose that there are $N$ clients in FL and each client $i\in\{1,2,\ldots,N\}$ is associated with a local dataset $D^c_i=\{(x^c_{i, 1}, y^c_{i, 1}), \ldots,(x^c_{i, n^c_{i}}, y^c_{i, n^c_{i}})\}$, where $n^c_i$ is the sample size of client $i$. Let $D=\{D^{c}_1,\cdots,D^{c}_N\}$ be the full dataset with sample size $n=\sum_{i=1}^N n^c_i$. Let $P^c_i$ and $P$ denote the data-generating distribution of each client data $D^c_i$ and whole data $D$ over $\mathcal{X} \times \mathcal{Y}$, respectively.
In general, the basic goal of FL is to learn a global model with parameters $\theta\in\Theta$ that performs well on distribution $P$ (in terms of average performance) without accessing the private data of clients.

\subsection{Unified Group Fairness on Federated Learning}
\textit{Group fairness} is usually defined via a \textbf{sensitive variable} $S$ and a \textbf{metric}, where the dataset is divided into different subgroups by $S$ and the requirement is an equalized metric for all protected subgroups. To extend group fairness to federated learning, we specify the sensitive variable $S$ and the metric with the concern of federated learning setting.

\subsubsection{Sensitive Variable}
Suppose that, for a given sensitive variable $S$, the full dataset $D$ is divided into $M$ groups by $S$: $D=\{D^g_1, D^g_2, \ldots, D^g_M\}$.
We let $P^g_i$ denote the data-generating distribution of $D^g_i$, and the distribution might be different across groups. 
Note that the sensitive variable $S$ can be defined not only as client index, but also as some sensitive attribute(s) including explicit attribute (e.g., gender), implicit attribute (e.g., domain) and target label. Moreover, the groups also can be the newly added clients.

\subsubsection{Fairness Metric}
In this paper, following the \textit{difference principle} on distributive justice and stability~\cite{rawls2001justice}, we view the performance of federated model as the resource which is supposed to be allocated into groups in a fair way.

Before formulating the problem formally, we first define $Disparity$ of a federated learning model across groups $\{D^g_i|i=1,2,\ldots,M\}$ as:
\begin{equation} \label{bias-acc}
\setlength\abovedisplayskip{4pt}
\setlength\belowdisplayskip{4pt}
\Scale[0.95]{Disparity = \sqrt{\frac{1}{M}\sum_{i=1}^{M}(Acc(D^g_i)-Avg\_Acc)^2}},
\end{equation}
where $Acc(D^g_i)$ is the predictive accuracy on group $D^g_i$, and $Avg\_Acc = \frac{1}{M}\sum_{i=1}^{M} Acc(D^g_i)$. 
In this work, we define fairness by $Disparity$, and the smaller $Disparity$, the fairer of a federated learning model.

\subsubsection{Unified Group Fairness}
Note that different sensitive variables $S$ correspond to different problems.
In this paper, we focus on the following three common cases: 1) \textbf{\emph{client-level fairness}}: $S$ specified as the index of existing clients; 2) \textbf{\emph{attribute-level fairness}}: $S$ specified as sensitive attribute; 3) \textbf{\emph{agnostic distribution fairness}}: $S$ specified as the index of potential clients with agnostic distribution. A model violating any of the above fairness definitions may lead to serious ethical problems in reality, so we propose a novel problem as below:
\begin{problem}[\textbf{Unified Group Fairness on FL}] \label{problem}
Let the sensitive variable $S$ be specified as the index of existing clients $c$, the protected attribute(s) $a$, and the index of newly added clients $ad$, respectively, then the dataset $D$ can be split into groups $\{D^c_k|k=1,2,\ldots,M^c\}$, $\{D^a_k|k=1,2,\ldots,M^a\}$, and $\{D^{ad}_k|k=1,2,\ldots,M^{ad}\}$, respectively. The task is to learn a federated model with small $Disparity$ across $\{D^c_k\}$, $\{D^a_k\}$, and $\{D^{ad}_k\}$ \textbf{\emph{simultaneously}}.
\end{problem}

\begin{remark}
In fact, a model with good fairness yet low performance is meaningless. $Avg\_Acc$ measures the overall performance of a model, but in many real applications (e.g., medical diagnosis, autonomous vehicles, credit evaluations and criminal justice), we are more concerned about FL model robustness. $Robustness$ in FL refers to the performance of the worst group with the following definition:
\begin{equation} \label{robustness-acc}
\setlength\abovedisplayskip{4pt}
\setlength\belowdisplayskip{4pt}
\Scale[0.95]{Robustness = \mathop{min}\limits_{i} Acc(D^g_i)}.
\end{equation}
In this paper, we focus on improving both the unified group fairness (in terms of $Disparity$) and $Robustness$ of the global model in FL.
\end{remark}

\section{Unified Fair Federated Framework}

In this section, we first introduce a general group fairness federated framework. Then, from the view of federated uncertainty set, we propose a group-based risk for unified group fairness on FL with theoretical analysis.

\subsection{DRO-based Group Fairness Federated Framework}
In order to encourage the federated model to have uniform performance over groups, a natural idea is to minimize a combination of the average risk and the variance of risk across different groups. Unfortunately, it is usually computationally intractable due to the variance term, especially in FL~\cite{duchi2016variance, mcmahan2017communication}.


To solve this problem, inspired by the techniques of distributionally robust optimization (DRO), we introduce a group fairness federated framework. Specifically, we first define an uncertainty set $\mathcal{Q}^g$ that contains any possible distribution formed by a mixture of the distributions of those groups. Then we introduce our proposed risk $\mathcal{R}_{\mathrm{group}}$ based on $\mathcal{Q}^g$:
\begin{equation} \label{group-dro-1} 
\setlength\abovedisplayskip{4pt}
\setlength\belowdisplayskip{4pt}
\Scale[0.95]{
\begin{aligned}
\mathcal{R}_{\mathrm{group}}(\theta)&:= \sup _{Q \in \mathcal{Q}^g}\left\{\mathbb{E}_{(x, y) \sim Q}[\ell(\theta, (x, y))] \right\},
\\
\mathcal{Q}^g &:= \{\sum\nolimits_{i=1}^M \lambda^g_{i}P^g_{i}: \bm{\lambda}^g \in \Delta_{M-1}\},
\end{aligned}}
\end{equation} 
where $\ell: \Theta \times (\mathcal{X} \times \mathcal{Y}) \rightarrow \mathbb{R}_{+}$ is a loss function, $M$ is the number of groups, $P^g_i$ is the distribution from which the samples of group $D^g_i$ are drawn, $\bm{\lambda}^g=(\lambda^g_{1}, \lambda^g_{2}, \ldots, \lambda^g_{M})$ is the vector of group weights, $\Delta_{M-1}$ is the standard $(M-1)$-simplex and uncertainty set $\mathcal{Q}^g$ contains any possible combinations of distributions of groups $\{P^g_i|i=1,2,\ldots,M\}$. Intuitive, by considering the potential distribution shifts over groups, the proposed risk (\ref{group-dro-1}) encourages the model to focus on those groups with high risk, so it helps to achieve group fairness. The following theorem supports our analysis.

\begin{theorem} Let $\mathcal{R}^g_{i}(\theta):=\mathbb{E}_{(x, y) \sim P^g_i}[\ell(\theta; (x, y))]$ be a risk defined on $i$-th group, $\bm{\lambda}^g \in \Delta_{m-1}$ be the group weights, $M$ be the total number of groups, $\Bar{\mathcal{R}}^g(\theta)$ be the average of group risks, $d_{i} := (\mathcal{R}^g_{i}(\theta)-\Bar{\mathcal{R}}^g(\theta))^2$ and $\mathrm{Var}(\mathcal{R}^g_{i}(\theta)):=\frac{1}{M}\sum_{i=1}^Md_{i}$ be the variance of group risks. If $\Vert M\bm{\lambda}^g-\bm{1}\Vert_2^2 \leq \mathop{min}\limits_{i}\{\frac{\sum_{i=1}^{M}d_{i}}{d_{i}}\}$, then there exists a constant $C>0$ such that
\begin{equation} \label{group-var}
\setlength\abovedisplayskip{4pt}
\setlength\belowdisplayskip{4pt}
\Scale[0.95]{
\mathcal{R}_{\mathrm{group}}(\theta)= \Bar{\mathcal{R}}^g(\theta) + C\sqrt{\mathrm{Var}_{i\in [M]}(\mathcal{R}^g_{i}(\theta))}.}
\end{equation}
\end{theorem}
See Appendix for the proof. The theorem above shows that our proposed risk $\mathcal{R}_{\mathrm{group}}$ can be viewed as a combination of the average risk that helps improve the average performance and the variance term that encourages the model to have a uniform performance across different groups. Therefore, minimizing the risk $\mathcal{R}_{\mathrm{group}}$ helps achieve group fairness.

\subsection{Construction of Uncertainty Set}
Note that the above framework is general because the uncertainty set $\mathcal{Q}^g$ can be arbitrarily selected as needed. If we select a client-level uncertainty set (i.e., $S$ specified as client index), the framework will degrade to AFL \cite{mohri2019agnostic}, a client-level fair FL method, with risk as below:
\begin{equation} \label{client-dro}
\setlength\abovedisplayskip{4pt}
\setlength\belowdisplayskip{4pt}
\Scale[0.95]{
\begin{aligned}
\mathcal{R}_{\mathrm{client}}(\theta)&:=\sup _{Q \in \mathcal{Q}^c}\left\{\mathbb{E}_{(x, y) \sim Q}[\ell(\theta, (x, y))] \right\}, \\
\mathcal{Q}^c &:= {\{\sum\nolimits_{i=1}^N \lambda^c_{i}P^c_{i}: \bm{\lambda}^c \in \Delta_{N-1}\}},
\end{aligned}}
\end{equation}
where $\bm{\lambda}^c=(\lambda^c_{1}, \lambda^c_{2}, \ldots, \lambda_{N})$ represents client weights, $P^c_i$ is the data-generating distribution of client $i$, and uncertainty set $\mathcal{Q}^c$ contains potential distribution shift over clients.

Similarly, if we select an attribute-level uncertainty set $\mathcal{Q}^a$, the framework will degrade to an attribute-level fair federated learning method with risk $\mathcal{R}_{\mathrm{attribute}}(\theta)$:
\begin{equation} \label{attribute-dro}
\setlength\abovedisplayskip{4pt}
\setlength\belowdisplayskip{4pt}
\Scale[0.95]{
\begin{aligned}
\mathcal{R}_{\mathrm{attribute}}(\theta)&:=\sup _{Q \in \mathcal{Q}^a}\left\{\mathbb{E}_{(x, y) \sim Q}[\ell(\theta, (x, y))] \right\}, \\
\mathcal{Q}^a &:= {\{\sum\nolimits_{i=1}^A \lambda^a_{i}P^a_{i}: \bm{\lambda}^a \in \Delta_{A-1}\}},
\end{aligned}}
\end{equation}
where $P^a_i$ denotes the data-generating distribution of samples from group $D^a_i$ (formed by protected attribute) and the sensitive attribute can take on $A$ values.

The essential issue of the client-level and attribute-level method lies in the construction of the uncertainty set, which is too small to deal with potential distribution shifts. Therefore, we should construct a wider uncertainty set that not only contains the distribution shifts considered by the client-level and attribute-level uncertainty set, but also contains the worse cases to adapt to newly added clients better.

As one possible way, we specify $S$ as the combination of the client index and the given sensitive attribute(s). Specifically, we divide the local dataset $D^c_i$ of client $i$ into subgroups $D^c_i=\{D^u_{i, 1}, D^u_{i, 2}, \ldots, D^u_{i, M_i}\}$ and consider the potential distribution shifts over them, where $M_i$ is the number of subgroups on client $i$. Suppose that the samples of $D^u_{i, k}$ are drawn from the distribution $P^u_{i, k}$. Then we define:
\begin{equation} \label{group-dro-2}
\setlength\abovedisplayskip{4pt}
\setlength\belowdisplayskip{4pt}
\Scale[0.95]{
\begin{aligned}
\mathcal{R}_{\mathrm{unified}}(\theta)&:= {\sup _{Q \in \mathcal{Q}^u}\left\{\mathbb{E}_{(x, y) \sim Q}[\ell(\theta, (x, y))] \right\}},
\\
\mathcal{Q}^u :=& {\{\sum\nolimits_{i=1}^N \sum\nolimits_{k=1}^{M_i} \lambda^u_{i, k}P^u_{i, k}: \bm{\lambda}^u \in \Delta_{M-1}\}},
\end{aligned}}
\end{equation} 
where $M=\sum_{i=1}^N M_i$ is the total number of subgroups and $\lambda_{i,k}^u$ is the weight of $k$-th subgroup on client $i$.

\begin{theorem} \label{theorem-subset}
Let $\hat{P}^c_i$, $\hat{P}^a_i$ and $\hat{P}^u_{i,k}$ be the empirical distributions over samples of local dataset $D^c_i$,  $D^a_i$, and group $D^u_{i,k}$ respectively, $\hat{\mathcal{Q}}^c:=\{\sum_{i=1}^N \lambda^c_{i}\hat{P}^c_{i}: \bm{\lambda}^c \in \Delta_{N-1}\}$ be the client-level uncertainty set, $\hat{\mathcal{Q}}^a:=\{\sum_{i=1}^A \lambda^a_{i}\hat{P}^a_{i}: \bm{\lambda}^a \in \Delta_{A-1}\}$ be the attribute-level uncertainty set, and
$\hat{\mathcal{Q}}^u:=\{\sum_{i=1}^N\sum_{k=1}^{M_i} \lambda^u_{i,k} \hat{P}^u_{i,k}: \bm{\lambda}^u \in \Delta_{M-1}\}$ be the unified uncertainty set. We have
\begin{equation}
\setlength\abovedisplayskip{4pt}
\setlength\belowdisplayskip{4pt}
\hat{\mathcal{Q}}^c \subseteq \hat{\mathcal{Q}}^u, \hat{\mathcal{Q}}^a \subseteq \hat{\mathcal{Q}}^u.
\end{equation}
Moreover, let
$\hat{\mathcal{R}}_{\mathrm{client}}(\theta)$, $\hat{\mathcal{R}}_{\mathrm{attribute}}(\theta)$ and $\hat{\mathcal{R}}_{\mathrm{unified}}(\theta)$ be the empirical risks based on uncertainty sets $\hat{\mathcal{Q}}^c$, $\hat{\mathcal{Q}}^a$ and $\hat{\mathcal{Q}}^u$ respectively.
We have
\begin{equation}
\setlength\abovedisplayskip{4pt}
\setlength\belowdisplayskip{4pt}
\hat{\mathcal{R}}_{\mathrm{client}}(\theta) \leq \hat{\mathcal{R}}_{\mathrm{unified}}(\theta), \hat{\mathcal{R}}_{\mathrm{attribute}}(\theta) \leq \hat{\mathcal{R}}_{\mathrm{unified}}(\theta).
\end{equation}
Furthermore, assume that ${\exists}$ $i$ and $k$, and $j$ ($j \neq i$) s.t. the attribute of samples from ${D}^u_{i,k}$ is same as ${D}^u_{j,k}$ but $\hat{P}^u_{i,k} \neq \hat{P}^u_{j,k}$, and $\hat{P}^u_{i,k} \neq \hat{P}^c_{l}$ for ${\forall}$ $l$, then we have
\begin{equation}
\setlength\abovedisplayskip{4pt}
\setlength\belowdisplayskip{4pt}
(\hat{\mathcal{Q}}^c \cup \hat{\mathcal{Q}}^a) \subsetneq \hat{\mathcal{Q}}^u.
\end{equation}
\end{theorem}
See Appendix for the proof. Theorem \ref{theorem-subset} shows that both client-level and attribute-level uncertainty sets are subsets of our proposed unified uncertainty set. Therefore, our proposed risk (\ref{group-dro-2}) provides an upper bound for both client-level risk and attribute-level risk, thereby optimizing it can guarantee client-level fairness and attribute-level fairness simultaneously. Actually, our proposed risk also considers the worse cases that the set union of client-level uncertainty set and attribute-level uncertainty set does not contain, which helps to deal with more complex distribution shifts.
However, it is difficult to theoretically and strictly demonstrate whether a certain method can guarantee agnostic distribution, since the distribution of newly added clients can be arbitrary. Therefore, we will show that the model trained by risk (\ref{group-dro-2}) can adapt to those newly added clients with different distributions well empirically (Section \ref{sec64}).

\subsection{Discussion: Individual-level Fairness}
Here we discuss an ideal but impractical way for unified group fairness by treating each sample as a group. Then, the framework will degenerate to an individual-level fairness method with risk:
\begin{equation} \label{sample-dro}
\setlength\abovedisplayskip{4pt}
\setlength\belowdisplayskip{4pt}
\Scale[0.95]{
\begin{aligned}
\mathcal{R}_{\mathrm{individual}}(\theta)&:=\sup _{Q \in \mathcal{Q}^{ind}}\left\{\mathbb{E}_{(x, y) \sim Q}[\ell(\theta, (x, y))] \right\}, \\ \mathcal{Q}^{ind} &:= \{Q | D_{f}\left(Q \| P\right) \leq r\},
\end{aligned}}
\end{equation}
where $P$ is the data-generating distribution of full dataset $D$, $f: \mathbb{R}_{+} \rightarrow \mathbb{R}$ is a convex function with $f(1)=0$, $D_{f}(Q \| P)=\int_{(\mathcal{X}, \mathcal{Y})} f\left(\frac{d Q}{d P}\right) d P$ is \textit{f-divergence} between distribution $Q$ and $P$ defined on $(\mathcal{X}, \mathcal{Y})$ and $r$ is radius of uncertainty set $\mathcal{Q}^{ind}$. Note that the individual-level method is more capable of modeling the agnostic distribution by constructing a wide uncertainty set around distribution $P$. Intuitively, it may help to guarantee unified group fairness.

Unfortunately, the uncertainty set defined on individual level is usually overwhelmingly large leading to the over pessimism problem in practice~\cite{hu2018does, sagawa2019distributionally, liu2021stable}. In fact, our proposed unified uncertainty set $\mathcal{Q}^u$ can be considered as a subset of $\mathcal{Q}^{ind}$ by imposing some structural constrains. Hence, By contrasting with risk (\ref{sample-dro}), our proposed risk (\ref{group-dro-2}) provides a relatively tight upper bound for both client-level risk and attribute-level risk, and help to overcome this pessimism.


\section{Algorithm and Optimization}

In this section, we propose a federated optimization algorithm named \textit{Federated Mirror Descent Ascent with Momentum Acceleration} (FMDA-M) to optimize the risk (\ref{group-dro-2}).

\subsection{Federated Gradient Descent Ascent}
We first introduce the empirical risk on $k$-th group of client $i$ defined as $f_{i, k}(\theta) := \mathbb{E}_{(x, y) \sim \hat{P}^u_{i, k}}[\ell(\theta ;(x, y))]$, where $\hat{P}^u_{i, k}$ is the empirical distribution over samples of group $D^u_{i, k}$.
Then, with the techniques of distributional robustness optimization~\cite{wiesemann2014distributionally, duchi2016variance, namkoong2016stochastic}, the problem of minimizing the risk in Eq.~(\ref{group-dro-2}) can be rewritten as:
\begin{equation} \label{group-ER-global}
\setlength\abovedisplayskip{4pt}
\setlength\belowdisplayskip{4pt}
\Scale[0.97]{
\min\limits _{\theta} \max\limits _{\bm{\lambda}^u\in{\Delta_{M-1}}} \{F(\theta, \bm{\lambda}^u):= {\sum\nolimits_{i=1}^{N}\sum\nolimits_{k=1}^{M_i} \lambda^u_{i, k} f_{i, k}(\theta)}\}},
\end{equation}

To solve this minimax problem, we can alternately optimize model parameters $\theta$ and weights $\bm{\lambda}^u$. Specifically, in each round $r$, the server samples a subset of clients $U^{(r)}$ according to weight $\lambda^c_i = \sum_{k=1}^{M_i} \lambda^u_{i, k}$.
Then client $i \in U^{(r)}$ samples data according to $\lambda^u_{i, k}$ and updates local models via stochastic gradient descent method at each local iteration $t=rE, rE+1, ..., (r+1)E-1$:
\begin{equation} \label{GD} 
\setlength\abovedisplayskip{4pt}
\setlength\belowdisplayskip{4pt}
\Scale[1]{
\theta_{i}^{(t+1)}=\theta_{i}^{(t)}-\eta \nabla_{\theta} l\left(\theta_{i}^{(t)} ; (x_i^{(t)}, y_i^{(t)})\right)}.
\end{equation}
After model aggregation
\begin{equation} \label{aggregation} 
\setlength\abovedisplayskip{4pt}
\setlength\belowdisplayskip{4pt}
\Scale[1]{
\theta^{(r+1)E}=\frac{1}{|U^{(r)}|}\sum\nolimits_{i\in U^{(r)}}\theta_i^{((r+1)E)}},
\end{equation}
clients calculate $\bm{v}^{(r)}=\nabla_{\bm{\lambda}^u}F(\theta, \bm{\lambda}^u)$, i.e. the loss of the global model $\theta^{(r+1)E}$ on their local subgroups $\{D_{i, k}^{u}\}$, and then the server updates group weight $\bm{\lambda}^u$ according to the following rule:
\begin{equation} \label{GAwithProj} 
\setlength\abovedisplayskip{4pt}
\setlength\belowdisplayskip{4pt}
\Scale[1]{
({\bm{\lambda}^u})^{(r+1)}=\prod\nolimits_{\Delta_{M-1}}(({\bm{\lambda}^u})^{(r)}+\gamma E \bm{v}^{(r)})},
\end{equation}
where $\gamma$ is stepsize and $\prod$ is projection operator based on Euclidean distance.

\subsection{Federated Mirror Descent Ascent}
The complexity of step (\ref{GAwithProj}) is usually $O(M log(M))$, so it is computationally expensive for the server when the number of clients or subgroups in FL is large. 
Moreover, the group weights generated by Euclidean distance-based projection are always too sparse, which may lead to instability of convergence in practice. 
In this paper, we adopt mirror gradient ascent method to get $({\bm{\lambda}^u})^{(r+1)}$ instead of step (\ref{GAwithProj}):
\begin{equation} \label{mirrorGA}
\setlength\abovedisplayskip{4pt}
\setlength\belowdisplayskip{4pt}
\Scale[0.95]{
\begin{aligned}
({\bm{\lambda}^u})^{(r+1)} & = 
\mathop{\arg\max}_{\bm{\lambda}\in \Delta_{M-1}} \{
F(\theta^{(r+1)E}, (\bm{\lambda}^u)^{(r)}) \\ + \langle E&\bm{v}^{(r)}, {(\bm{\lambda}^u)^{(r)}}-{\bm{\lambda}}\rangle - \frac{1}{\gamma} D_{h}\left(\bm{\lambda} \| (\bm{\lambda}^u)^{(r)}\right)
\},
\end{aligned}}
\end{equation}
where $\gamma>0$ is stepsize and $D_{h}(\cdot,\cdot)$ is a Bregman distance based on a convex function $h(\cdot)$. The first two terms of the above subproblem (\ref{mirrorGA}) is a linear approximation of $F(\theta^{(r+1)E}, \bm{\lambda})$, and the last term is a Bregman distance between $\bm{\lambda}$ and $(\bm{\lambda}^u)^{(r)}$. Note that we can choose a suitable function $h(\cdot)$ to solve the problems mentioned above. Actually, for the negative entropy function $h(\bm{x})=\sum_{i=1}^{n}\left(x_{i} \ln x_{i}-x_i\right)$, the subproblem (\ref{mirrorGA}) has an explicit solution:
\begin{equation} \label{ExplicitSolution}
\setlength\abovedisplayskip{4pt}
\setlength\belowdisplayskip{4pt}
(\lambda^u_{i,k})^{(r+1)} = \frac{(\lambda^u_{i,k})^{(r)}e^{\gamma E v^{(r)}_{i,k}}}{\sum^{N}_{i=1}\sum^{M_i}_{k=1}(\lambda^u_{i,k})^{(r)}e^{\gamma E v^{(r)}_{i,k}}}.
\end{equation}
By replacing step (\ref{GAwithProj}) with step (\ref{ExplicitSolution}), the calculations for updating group weights are almost free. Moreover, the weights calculated by step (\ref{ExplicitSolution}) are smoother, which is important for stability of convergence. We call the proposed algorithm as Federated Mirror Descent Ascent (FMDA).

\subsection{Momentum Acceleration}
Since communication costs are the principal constraint in FL \cite{mcmahan2017communication}, we explore to improve the convergence rate of the above federated algorithm by leveraging momentum acceleration techniques \cite{nesterov1983method, li2017convergence, ochs2018local}. Specifically, denoting the results obtained in step (\ref{aggregation}) and step (\ref{ExplicitSolution}) as $\widetilde{{\theta}}^{(r+1)E}$ and $(\widetilde{\bm{\lambda}}^u)^{(r+1)}$ respectively, we update model parameters as below:
\begin{equation} \label{momentum-x}
\setlength\abovedisplayskip{4pt}
\setlength\belowdisplayskip{4pt}
\theta^{(r+1)E}=\widetilde{\theta}^{(r+1)E}+\beta_{\theta}(\widetilde{\theta}^{(r+1)E}-\widetilde{\theta}^{rE}),
\end{equation}
and update group weights according to the following rule:
\begin{equation} \label{momentum-y}
\setlength\abovedisplayskip{4pt}
\setlength\belowdisplayskip{4pt}
(\bm{\lambda}^u)^{(r+1)}=(\bm{\lambda}^u)^{(r)}+\beta_{\bm{\lambda}}((\widetilde{\bm{\lambda}}^u)^{(r+1)}-(\bm{\lambda}^u)^{r}),
\end{equation}
where $\beta_{\theta}$ and $\beta_{\bm{\lambda}}$ are momentum coefficients. The second terms of step (\ref{momentum-x}) and step (\ref{momentum-y}) are momentum terms, which contains historical gradient information and helps speed up the convergence of our algorithm.

The details of our algorithm Federated Mirror Descent Ascent with Momentum Acceleration (FMDA-M) are formally presented in Appendix.

\subsection{Convergence Analysis}
Here we give the convergence rate of the proposed FMDA-M algorithm in convex-concave setting.
\begin{theorem}
Suppose that each function $f_{i,k}$ is convex and L-smooth, global function $F$ is linear in $\bm{\lambda}$ and L-smooth and the gradient w.r.t. $\theta$ and $\lambda$, model parameters $\theta$, the variance of stochastic gradient method w.r.t. $\theta$ and $\lambda$ are bounded. If we optimize (\ref{group-ER-global}) using FMGDA-M algorithm with local iterations $E=O(T^{\frac{1}{4}})$, learning rate for model parameters $\eta=O(T^{-\frac{1}{2}})$ and stepsize for group weights $\gamma=O(T^{-\frac{1}{2}})$, and then it holds that
\begin{equation}
\setlength\abovedisplayskip{4pt}
\setlength\belowdisplayskip{4pt}
\varepsilon_T \leq O(T^{-\frac{1}{2}}).
\end{equation}
\end{theorem}
See Appendix for the proof.

\section{Experiments}

In this section, we validate the effectiveness of our method on three datasets with different distribution shift.

\subsection{Experimental Setup}
\textbf{Federated Datasets.}
(1) Fashion-MNIST (FM) dataset \cite{xiao2017fashion}: FM is a classical image classification dataset containing 60,000 training examples with 10 categories. For FM, we set the target label as the sensitive attribute and consider \textbf{\textit{label distribution shift}} across clients. As shown in Figure~\ref{fig-distribution}, we consider 4 different degrees of Non-IID settings: (a) IID, (b) weakly Non-IID, (c) strongly Non-IID, (d) extremely Non-IID. We run our algorithm and compared baselines on FM dataset with logistic regression model. (2) Digit-Five (D5) dataset \cite{xu2018deep, peng2019moment, zhao2020multi}: D5 includes digit images with 10 categories sampled from 5 domains. For D5, we set the domain (i.e. data collection source) as the sensitive attribute and consider 4 different degrees of \textbf{\textit{feature distribution shift}} across clients. We use a 2-layer CNN with a linear classifier. (3) UCI Adult dataset\cite{Dua:2019}: Adult is a census dataset with 32,561 examples, and each sample has 14 features (including race, gender and so on) and a target label indicating whether the income is greater or less than \$50K. For Adult, we set gender and income as the protected attributes and consider 4 different degrees of \textbf{\textit{unbalance}} setting (i.e. the amount of data varies greatly across both clients and attributes), which is hard to avoid in real FL scenario. We use a logistic regression model to predict the income. More details in Appendix.

\begin{figure}[t]
\centering                         
\subfigure[IID]{
\begin{minipage}[t]{0.45\linewidth}
\centering  
\includegraphics[width=1in]{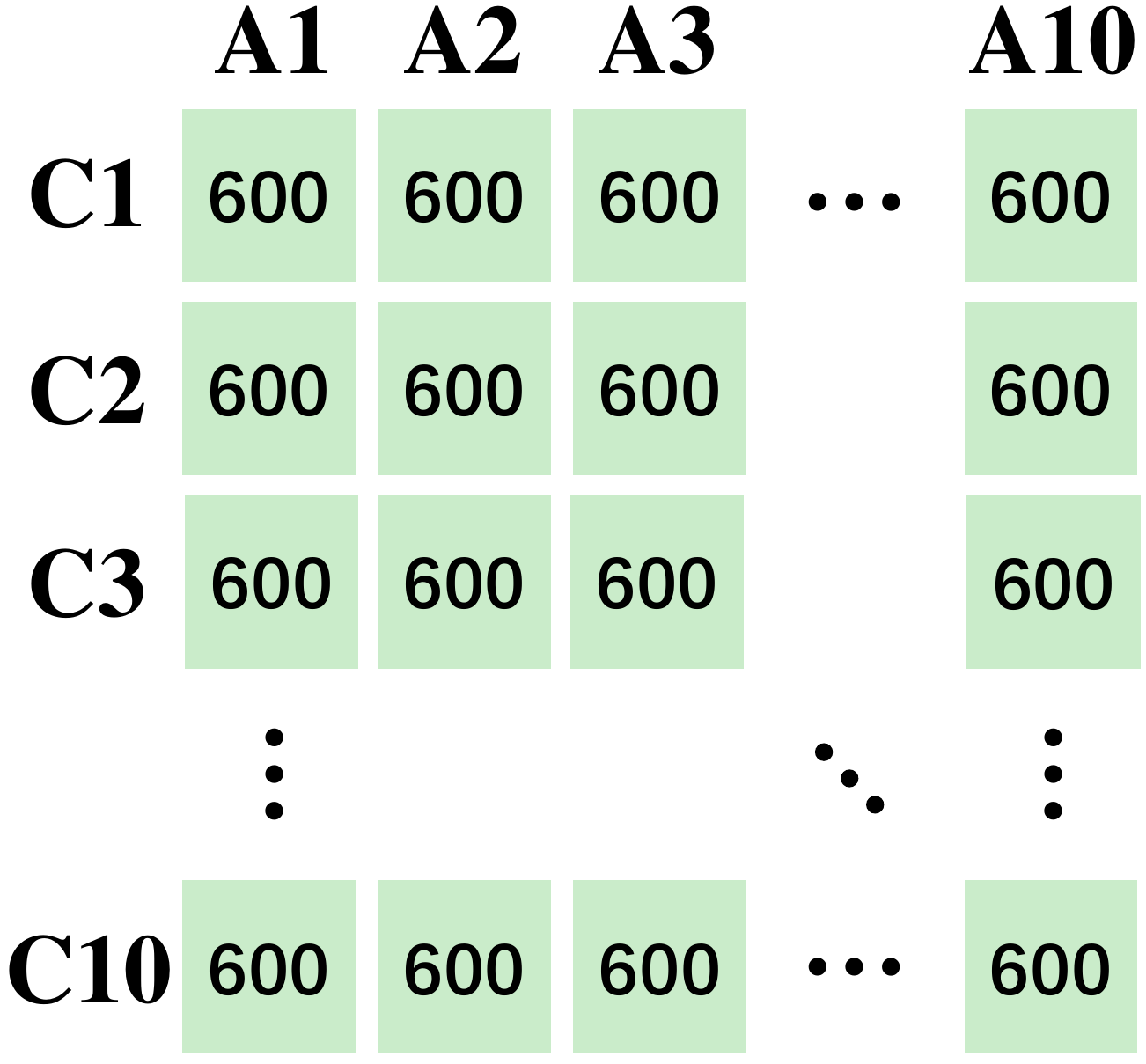}             
\end{minipage}}
\subfigure[Weakly Non-IID]{
\begin{minipage}[t]{0.45\linewidth}
\centering  
\includegraphics[width=1in]{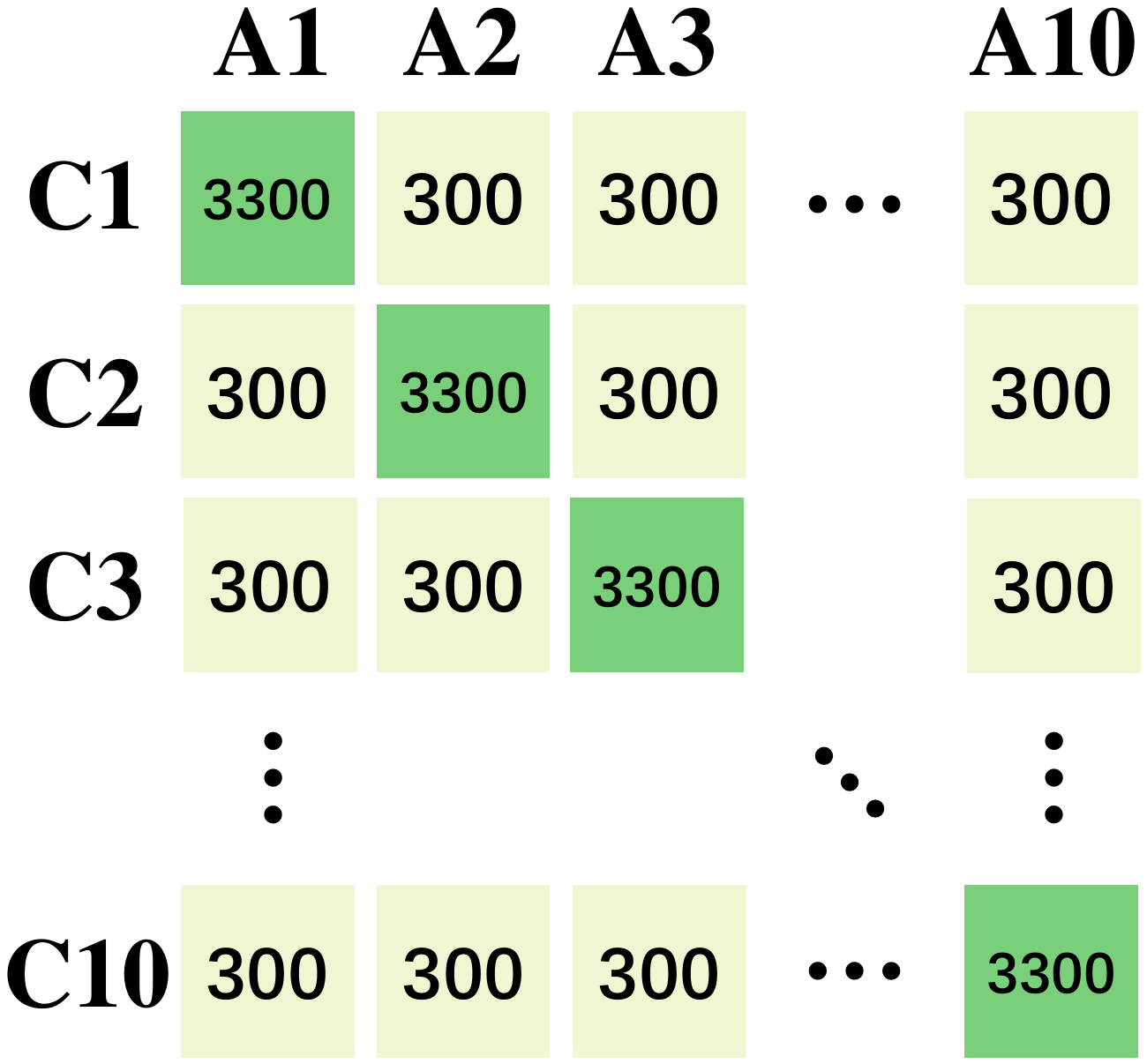}             
\end{minipage}}
\subfigure[Strongly Non-IID]{
\begin{minipage}[t]{0.45\linewidth}
\centering  
\includegraphics[width=1in]{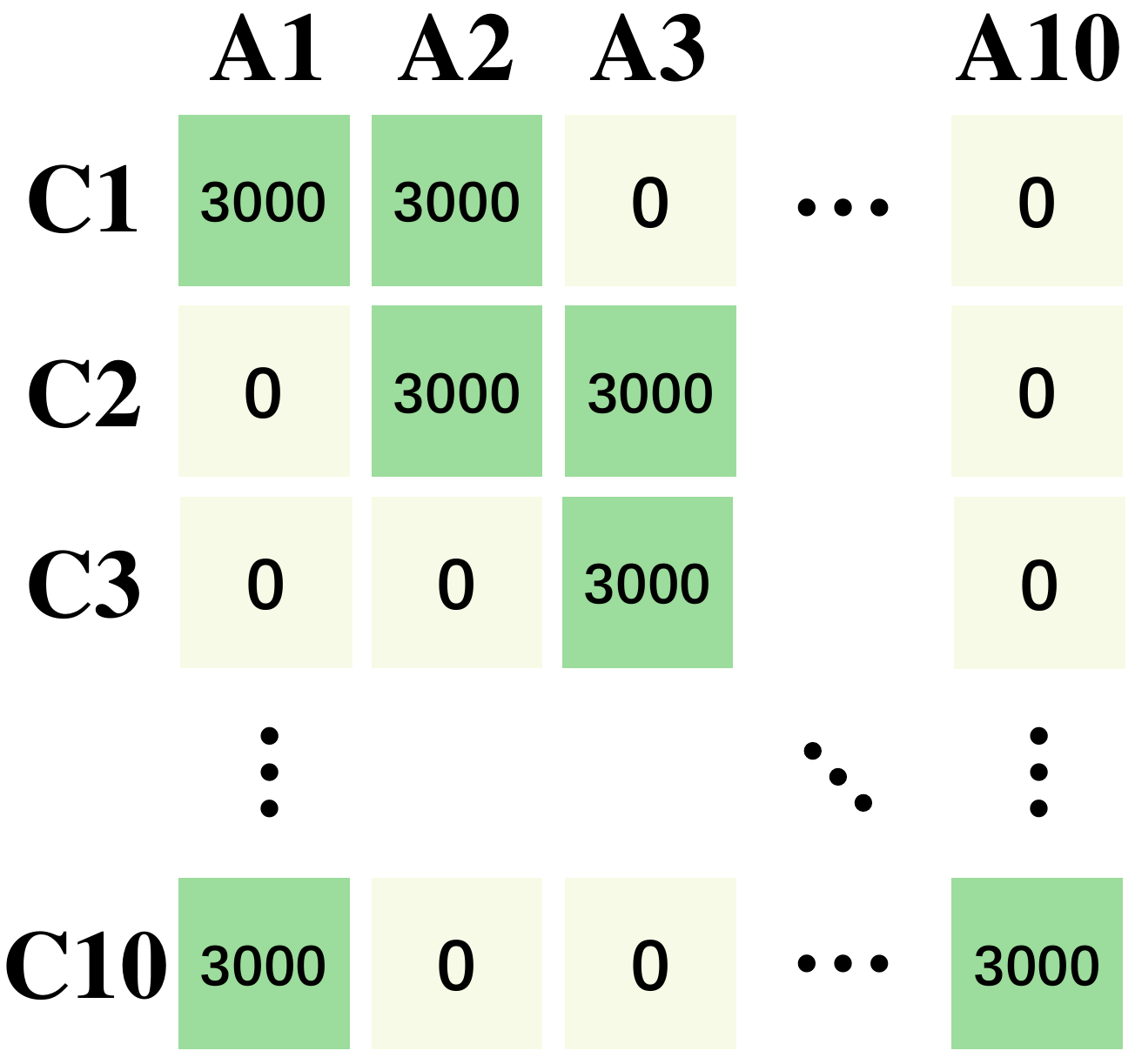}      

\end{minipage}}
\subfigure[Extremely Non-IID]{
\begin{minipage}[t]{0.45\linewidth}
\centering  
\includegraphics[width=1in]{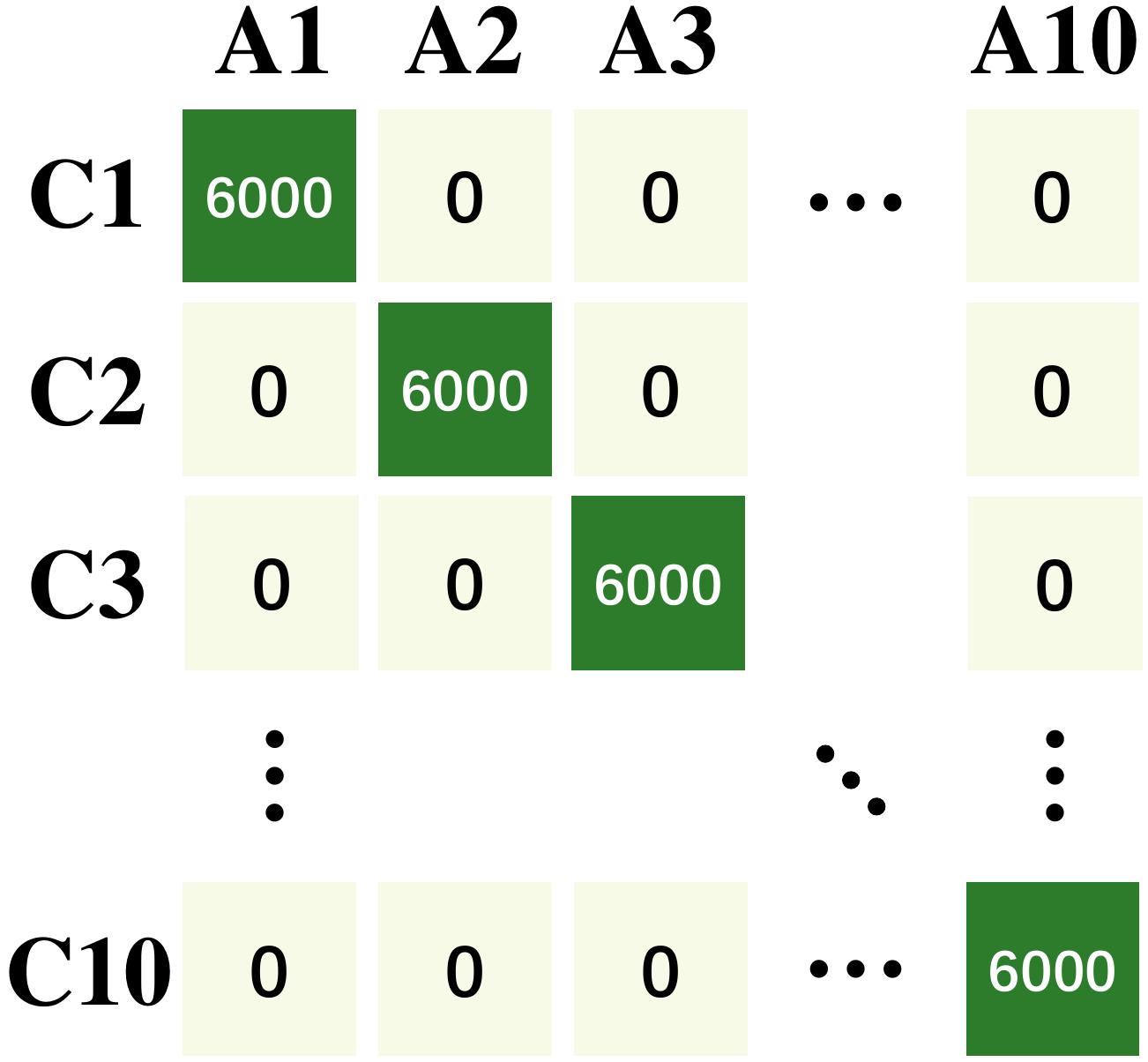}\label{setting3}  
\end{minipage}}
\vspace{-0.5em}
\caption{Various training distribution settings on Fashion-MNIST. The number of C$i$-A$k$ demotes the sample size of group $D^u_{i,k}$.}
\vspace{-1.3em}
\label{fig-distribution}                       
\end{figure}

\begin{table*}[t]
\caption{Experimental results of attribute-level fairness.}
\centering
\scalebox{0.7}{
\begin{tabular}{c|c|cccc|cccc}
\hline
 & Metrics & \multicolumn{4}{c|}{Robustness: $Robustness$ (\%)} & \multicolumn{4}{c}{Fairness: $Disparity$} \\ \hline
\multirow{2}{*}{Dataset} & \multirow{2}{*}{Method} & \multirow{2}{*}{IID} & Weakly & Strongly & Extremely & \multirow{2}{*}{IID} & Weakly & Strongly & Extremely \\
 &  &  & Non-IID & Non-IID & Non-IID &  & Non-IID & Non-IID & Non-IID \\ \hline
\multirow{2}{*}{} & Centralized ERM & \multicolumn{4}{c|}{56.67±0.30} & \multicolumn{4}{c}{0.122±0.001} \\
 & Centralized DRO & \multicolumn{4}{c|}{\textbf{68.98±0.35}} & \multicolumn{4}{c}{\textbf{0.081±0.002}} \\
 {Fashion} & FedAvg & 55.27±0.33 & 53.41±0.44 & 41.63±0.28 & 40.60±0.20 & 0.121±0.001 & 0.123±0.003 & 0.151±0.003 & 0.152±0.003 \\
 {MNIST}& DRFA & 61.90±0.53 & 61.50±0.87 & 56.11±0.44 & 65.63±0.47 & 0.107±0.001 & 0.109±0.002 & 0.119±0.003 & \ 0.101±0.001 \\
 \multirow{2}{*}{} & IndA & 55.04±0.85 & 58.46±0.99 & 45.91±0.81 & 51.35±0.61 & 0.139±0.006 & 0.136±0.006 & 0.145±0.006 & 0.130±0.007 \\
 & FMDA-M (Ours) & \textbf{68.06±0.42} & \textbf{66.60±0.68} & \textbf{68.31±0.38} & \textbf{67.72±0.23} & \textbf{0.082±0.002} & \textbf{0.086±0.002} & \textbf{0.085±0.002} & \textbf{0.087±0.002} \\ \hline
{} & Centralized ERM & \multicolumn{4}{c|}{83.07±0.48} & \multicolumn{4}{c}{0.059±0.001} \\
{} & Centralized DRO & \multicolumn{4}{c|}{\textbf{84.48±0.47}} & \multicolumn{4}{c}{\textbf{0.046±0.002}} \\
{Digit} & FedAvg & 81.72±0.22 & 81.13±0.35 & 81.60±0.33 & 78.51±0.64 & 0.063±0.002 & 0.065±0.002 & 0.053±0.002 & 0.073±0.003 \\
{Five} & DRFA & 81.22±0.34 & 82.02±0.33 & 82.05±0.50 & 81.08±0.32 & 0.064±0.001 & 0.062±0.001 & 0.058±0.001 & 0.060±0.001 \\
{} & IndA & 81.14±0.28 & 81.03±0.38 & 81.14±0.21 & 78.99±0.51 & 0.067±0.003 & 0.067±0.002 & 0.065±0.002 & 0.070±0.004 \\
{} & FMDA-M (Ours) & \textbf{85.10±0.23} & \textbf{83.91±0.24} & \textbf{84.20±0.22} & \textbf{81.98±0.36} & \textbf{0.043±0.001} & \textbf{0.044±0.002} & \textbf{0.054±0.001} & \textbf{0.051±0.001} \\ \hline
\multirow{6}{*}{Adult} & Centralized ERM & \multicolumn{4}{c|}{20.85±0.47} & \multicolumn{4}{c}{0.318±0.003} \\
 & Centralized DRO & \multicolumn{4}{c|}{\textbf{70.42±0.56}} & \multicolumn{4}{c}{\textbf{0.031±0.003}} \\
 & FedAvg & 20.26±0.29 & 25.82±0.49 & 20.70±0.42 & 66.04±0.36 & 0.322±0.002 & 0.290±0.002 & 0.323±0.003 & 0.063±0.001 \\
 & DRFA & 20.19±0.32 & 31.15±0.55 & 20.62±0.55 & 67.85±0.61 & 0.322±0.003 & 0.264±0.003 & 0.324±0.003 & 0.046±0.002 \\
 & IndA & 21.60±0.58 & 25.69±0.61 & 21.90±0.63 & 65.21±0.68 & 0.315±0.004 & 0.292±0.003 & 0.324±0.005 & 0.079±0.004 \\
 & FMDA-M (Ours) & \textbf{70.20±0.35} & \textbf{71.16±0.41} & \textbf{70.74±0.44} & \textbf{70.57±0.48} & \textbf{0.031±0.002} & \textbf{0.030±0.002} & \textbf{0.031±0.002} & \textbf{0.032±0.001} \\ \hline
\end{tabular}
}
\label{table-group}
\end{table*}

\begin{table*}[t]
\caption{Experimental results of client-level fairness.}
\centering
\scalebox{0.7}{
\begin{tabular}{c|c|cccc|cccc}
\hline
 & Metrics & \multicolumn{4}{c|}{Robustness: $Robustness$ (\%)} & \multicolumn{4}{c}{Fairness: $Disparity$} \\ \hline
\multirow{2}{*}{Dataset} & \multirow{2}{*}{Method} & \multirow{2}{*}{IID} & Weakly & Strongly & Extremely & \multirow{2}{*}{IID} & Weakly & Strongly & Extremely \\
 &  &  & Non-IID & Non-IID & Non-IID &  & Non-IID & Non-IID & Non-IID \\ \hline
{} & FedAvg & 84.34±0.47 & 70.24±0.64 & 65.92±0.47 & 44.04±0.38 & \textbf{0.004±0.001} & 0.059±0.002 & 0.091±0.006 & 0.146±0.004 \\
{Fashion} & DRFA & \textbf{84.39±0.45} & 71.25±0.88 & 70.93±0.80 & 66.54±0.68 & 0.004±0.001 & 0.055±0.003 & 0.064±0.002 & 0.099±0.002 \\
{MNIST} & IndA & 83.82±0.71 & 72.75±1.03 & 69.69±1.09 & 53.00±1.06 & 0.004±0.001 & 0.057±0.003 & 0.081±0.008 & 0.125±0.006 \\
{} & FMDA-M (Ours) & 81.05±0.36 & \textbf{72.78±0.63} & \textbf{73.82±0.59} & \textbf{68.45±0.28} & 0.005±0.001 & \textbf{0.039±0.002} & \textbf{0.042±0.002} & \textbf{0.069±0.002} \\ \hline
{} & FedAvg & \textbf{90.58±0.30} & 84.79±0.42 & 84.42±0.33 & 74.97±0.40 & 0.001±0.000 & 0.035±0.001 & 0.054±0.002 & 0.083±0.005 \\
{Digit} & DRFA & 90.23±0.34 & 86.02±0.59 & 83.23±0.41 & 79.86±0.30 & 0.001±0.000 & 0.031±0.001 & 0.052±0.002 & 0.060±0.002 \\
{Five} & IndA & 89.61±0.58 & 84.13±0.77 & 82.97±0.71 & 77.64±0.99 & 0.002±0.001 & 0.036±0.002 & 0.056±0.003 & 0.074±0.004 \\
{} & FMDA-M (Ours) & 87.81±0.25 & \textbf{86.18±0.36} & \textbf{85.45±0.37} & \textbf{81.34±0.34} & \textbf{0.001±0.000} & \textbf{0.023±0.001} & \textbf{0.037±0.002} & \textbf{0.048±0.003} \\ \hline
\multirow{4}{*}{Adult} & FedAvg & \textbf{82.11±0.43} & 68.61±0.48 & 77.95±0.49 & 66.89±0.44 & 0.004±0.001 & 0.091±0.004 & 0.055±0.003 & 0.067±0.002 \\
 & DRFA & 82.01±0.33 & 69.77±0.55 & \textbf{78.29±0.57} & 68.53±0.60 & 0.004±0.001 & 0.082±0.003 & 0.053±0.002 & 0.051±0.002 \\
 & IndA & 81.31±0.61 & 68.53±0.61 & 77.98±0.74 & 64.72±0.90 & 0.003±0.001 & 0.090±0.005 & 0.052±0.004 & 0.074±0.004 \\
 & FMDA-M (Ours) & 75.26±0.52 & \textbf{74.03±0.47} & 74.98±0.43 & \textbf{71.84±0.48} & \textbf{0.003±0.000} & \textbf{0.012±0.001} & \textbf{0.010±0.001} & \textbf{0.029±0.002} \\ \hline
\end{tabular}
}
\label{table-client}
\end{table*}

\begin{figure*}[htbp]
\centering
\subfigure[Fashion-MNIST]{
\begin{minipage}[t]{0.3\linewidth}
\centering 
\includegraphics[width=2in]{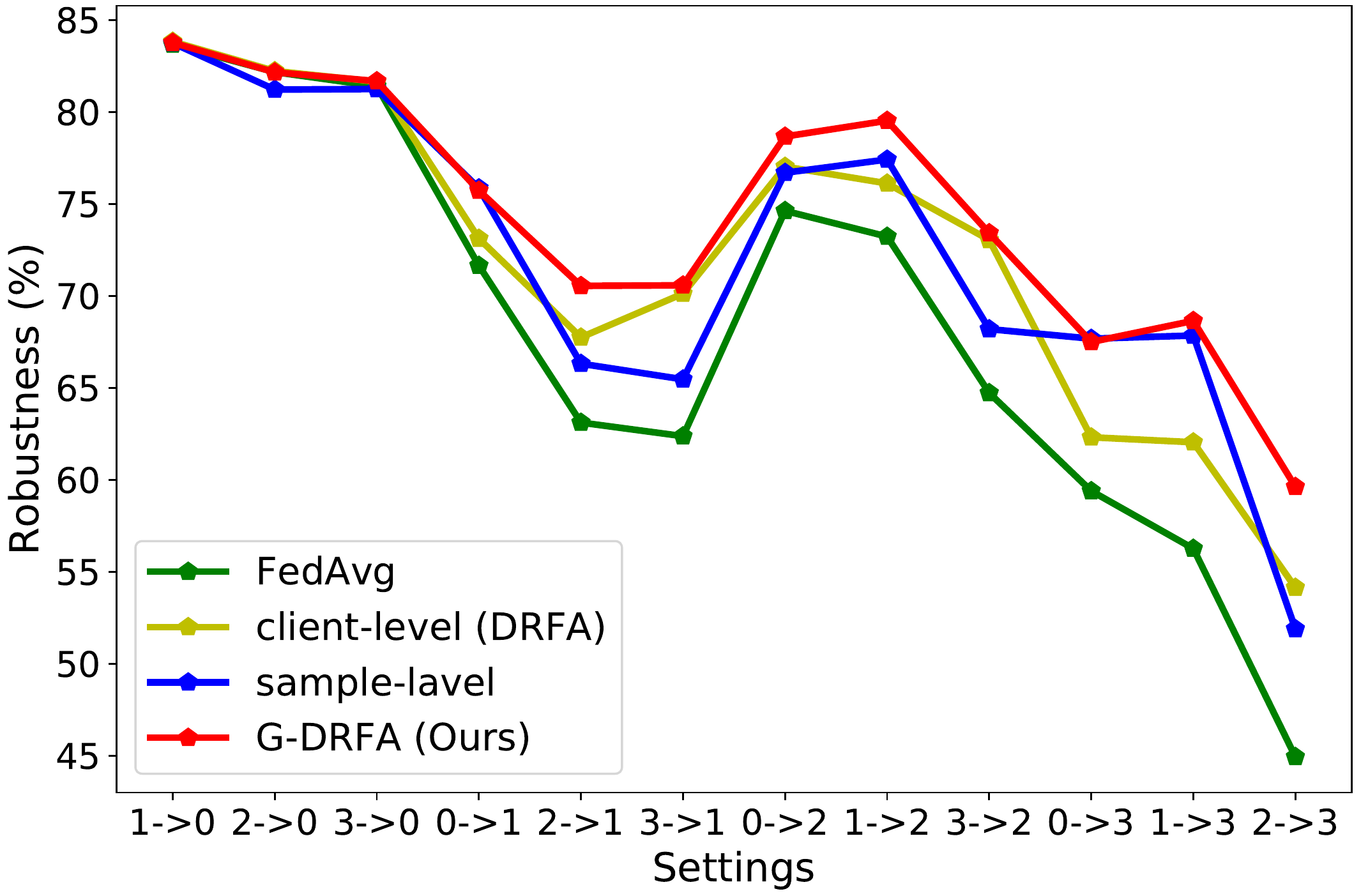}
\includegraphics[width=2in]{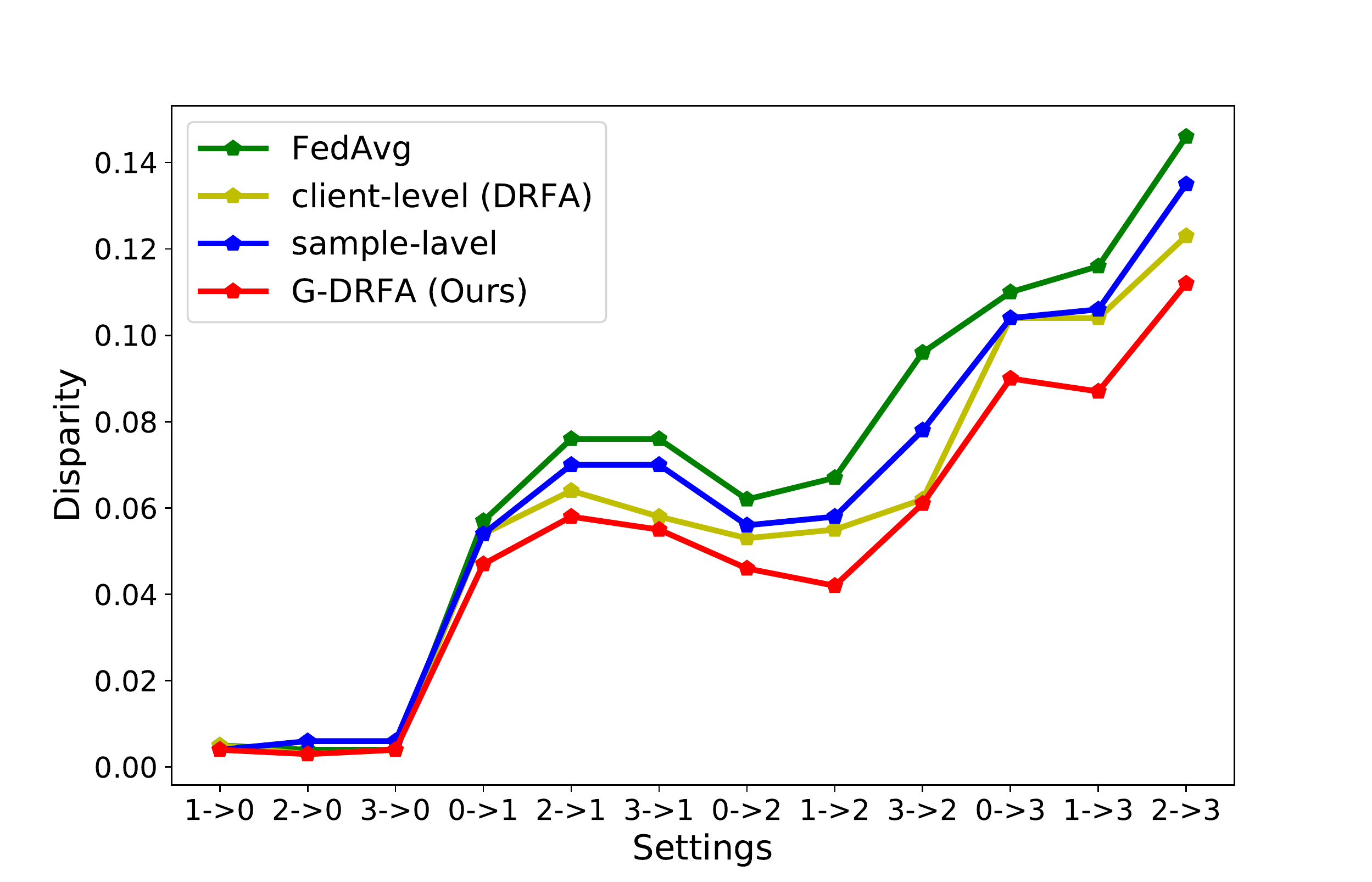}
\label{fig-agnostic-FM}
\end{minipage}}
\subfigure[Digit-Five]{
\begin{minipage}[t]{0.3\linewidth}
\centering 
\includegraphics[width=2in]{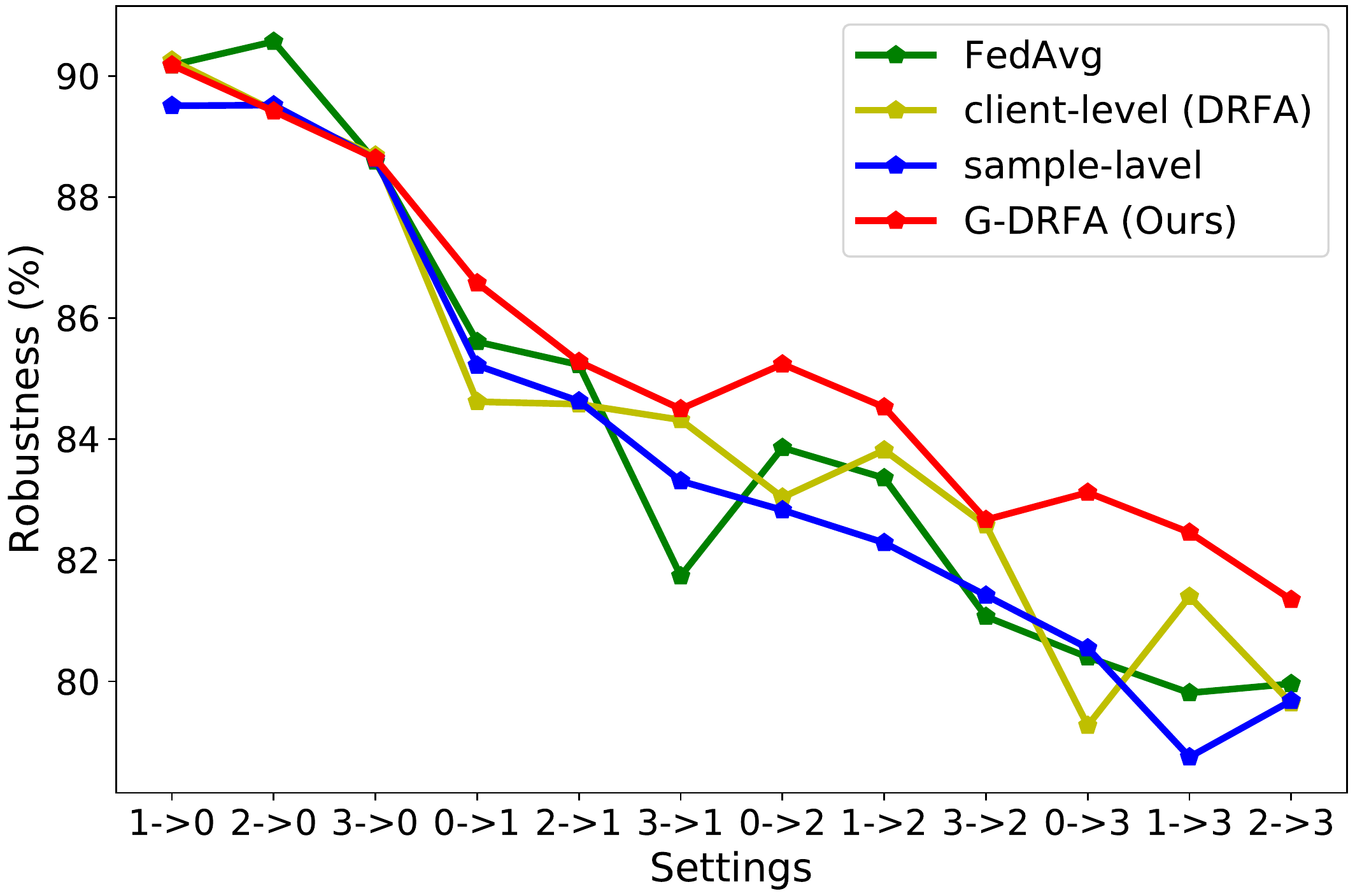} 
\includegraphics[width=2in]{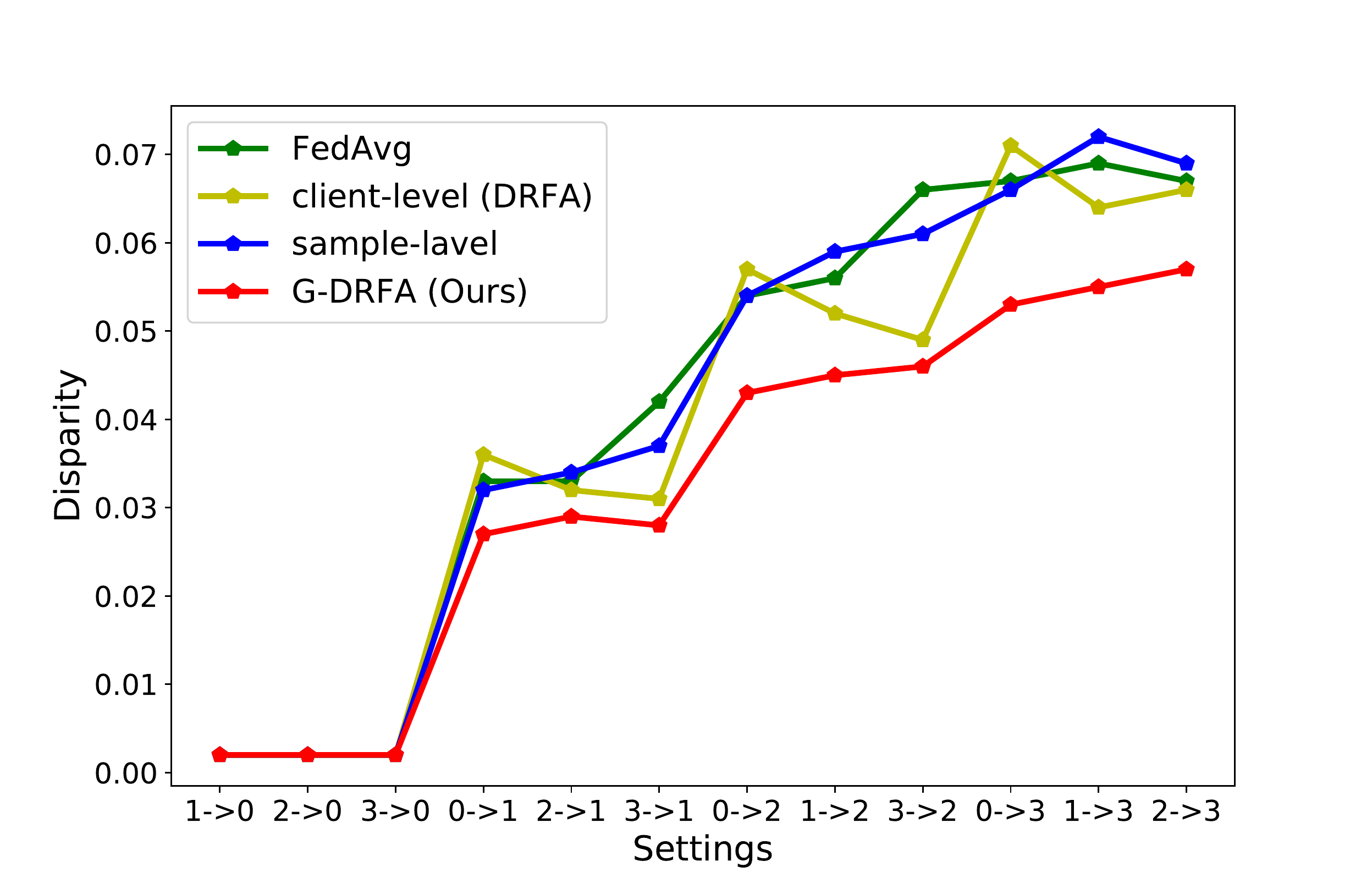} 
\label{fig-agnostic-D5} 
\end{minipage}}
\subfigure[Adult]{
\begin{minipage}[t]{0.3\linewidth}
\centering 
\includegraphics[width=2in]{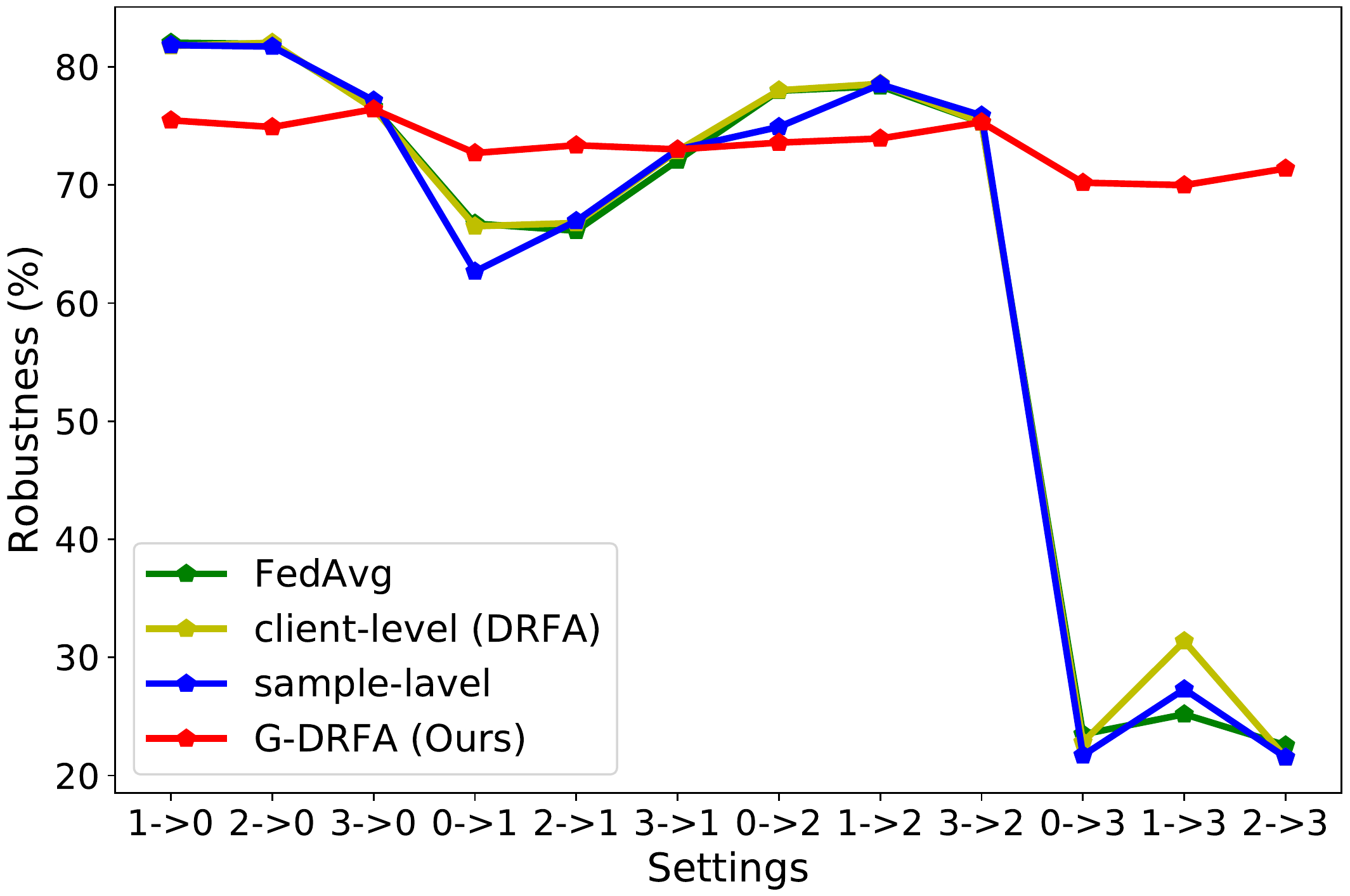}  
\includegraphics[width=2in]{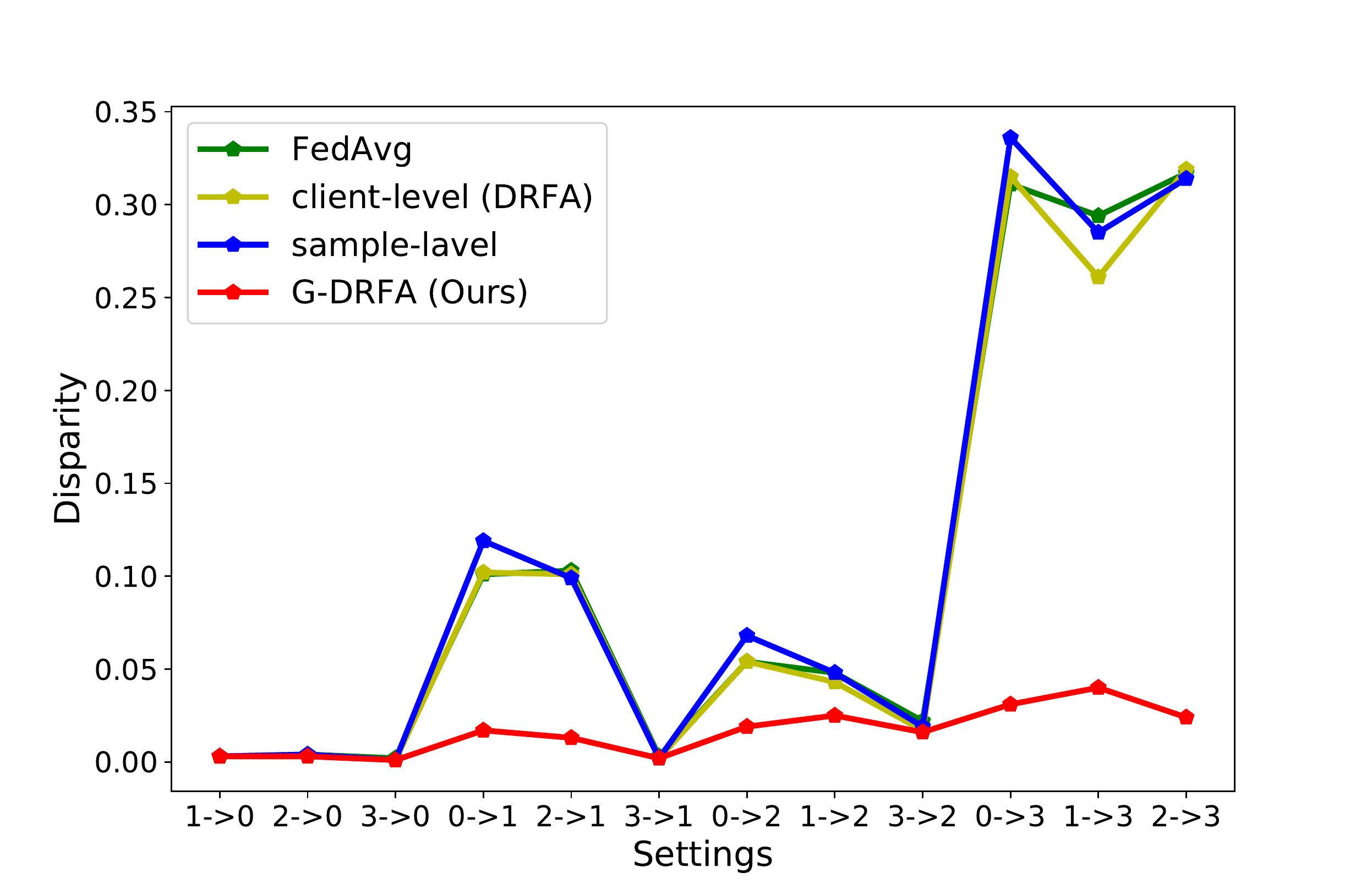}  
\label{fig-agnostic-Adult}
\end{minipage}}
\vspace{-1em}
\caption{Experimental results of robustness (the top row) and fairness (the bottom row) on agnostic distribution. IID, weakly Non-IID, strongly Non-IID and extremely Non-IID are denoted by setting 0, setting 1, setting 2 and setting 3, respectively. The coordinate $i\rightarrow j$ of horizontal axis means that we train a federated learning model in setting $i$ and test it in setting $j$.} 
\vspace{-1em}
\label{fig-agnostic}
\end{figure*}

\textbf{Evaluation Metrics.}
We evaluate models' unified group fairness on three levels: (1) attribute-level fairness, where the groups are formed by sensitive attributes; 2) client-level fairness, where the groups are formed by existing client index; and 3) agnostic distribution fairness, where the groups are formed by newly added client index.
In each kind of fairness, we use $Disparity$ in Eq.~(\ref{bias-acc}) to measure the degree of fairness, and use $Robustness$ in Eq.~(\ref{robustness-acc}) to measure the robustness. The average accuracy of the models are similar/comparable for all algorithms, and reported in appendix.

\textbf{Baselines.}
We compare the proposed FMDA-M algorithm with the following baselines:
(i) FedAvg~\cite{mcmahan2017communication}: FedAvg is a commonly used algorithm in FL, which minimizes an average risk.
(ii) DRFA~\cite{deng2021distributionally}: DRFA is a client-level federated optimization algorithm by minimizing the risk (\ref{client-dro}), which can be viewed as an improvement on the AFL~\cite{mohri2019agnostic}.
(iii) Individual-level Algorithm (denoted as IndA for convenience): We use the same algorithm as FMDA-M to solve the individual-level problem with objective (\ref{sample-dro}) as a compared baseline.
(iv) Centralized Algorithms: We also evaluate the attribute-level fairness of models trained by the traditional ERM and DRO~\cite{sagawa2019distributionally} in centralized setting, which requires a unified available training set.

\subsection{Results of Attribute-level Fairness}
We evaluate the attribute-level fairness of models trained by FMDA-M and compared baselines. The results are reported in Table \ref{table-group}. From the results, we observe that FMDA-M outperforms baselines on three datasets, in terms of both the metric $Disparity$ and $Robustness$. As we analyzed in the previous section, the client-level method is not flexible enough to deal with distribution shifts over attributes, and the individual-level method is too conservative to perform well in practice. By constructing an appropriate uncertainty set, FMDA-M achieves good performance which is very similar to centralized DRO, even in Non-IID settings.

The results on Adult dataset demonstrate that the unbalance of dataset is a great challenge for training a fair model, especially in FL where the data distribution of each client is unknown to others. We find that the overfitting appears in FedAvg, DRFA and individual-level method due to the jumbo sample size of the low-income male group. By contrast, our proposed FMDA-M samples from each subgroup according to the group weights, thus overcoming this challenge.

Note that FMDA-M also outperforms DRFA in extremely Non-IID setting, where the unified group fairness optimized in our FMDA-M is exactly the client-level fairness optimized in DRFA, as shown in Fig.~\ref{setting3}. The reason for this is that DRFA follows Eq.~(\ref{GAwithProj}) to update weights and the projection operator usually leads to the very hard weights, which may affect the stability of algorithm. By contrast, our algorithm adopts Eq.~(\ref{ExplicitSolution}) to generate smoother weights and it helps the model to converge.

\subsection{Results of Client-level Fairness}
We evaluate the client-level fairness of models trained by different algorithm and report the results in Table \ref{table-client}. We observe that FMDA-M is able to guarantee the accuracy of the worst-performing client and decrease $Disparity$ in most training distribution settings. We also find that, unlike other methods of which the performance decreases significantly with increasing degree of Non-IID, FMDA-M shows extremely stable performance under various settings, which is thanks to the weights update rule (\ref{ExplicitSolution}) we adopt.

We note that $Robustness$ of FMDA-M is slightly lower than FedAvg and DRFA in IID setting, because the distributions of clients are very similar and the $Robustness$ will degrade to average accuracy, which is in line with the optimization objective of FedAvg and DRFA. Indeed, our FMDA-M significantly improves client-level fairness in Non-IID settings (more challenging and more common in reality), though occasionally with a small performance sacrifice in IID setting.

\subsection{Results of Agnostic Distribution Fairness} \label{sec64}

To evaluate the agnostic distribution fairness, we simulate the newly added clients as follows: we train each federated model in one of the the training distribution settings (e.g., IID setting), but test under other three settings (e.g., weakly Non-IID, strongly Non-IID and extremely Non-IID settings) where the distributions of clients are different and agnostic from the existing clients.

The results of agnostic distribution fairness are shown in Figure~\ref{fig-agnostic}. We find that our FMDA-M outperforms compared baselines in terms of both $Robustness$ and $Disparity$ in most cases, which illustrates that our FMDA-M is better adapted to new distributions. As we state before, the proposed FMDA-M considers a larger uncertainty set but with appropriate degrees of freedom, so the model trained by FMDA-M can deal with kinds of new distributions. However, the resulting model can be overly pessimistic when the radius of uncertainty set is too large. Besides, the individual-level method is hard to optimize, and that is why the individual-level method does not perform well.

\begin{figure}[t]
\centering                         
\subfigure[Robustness]{
\begin{minipage}[t]{0.48\linewidth}
\includegraphics[width=1.5in]{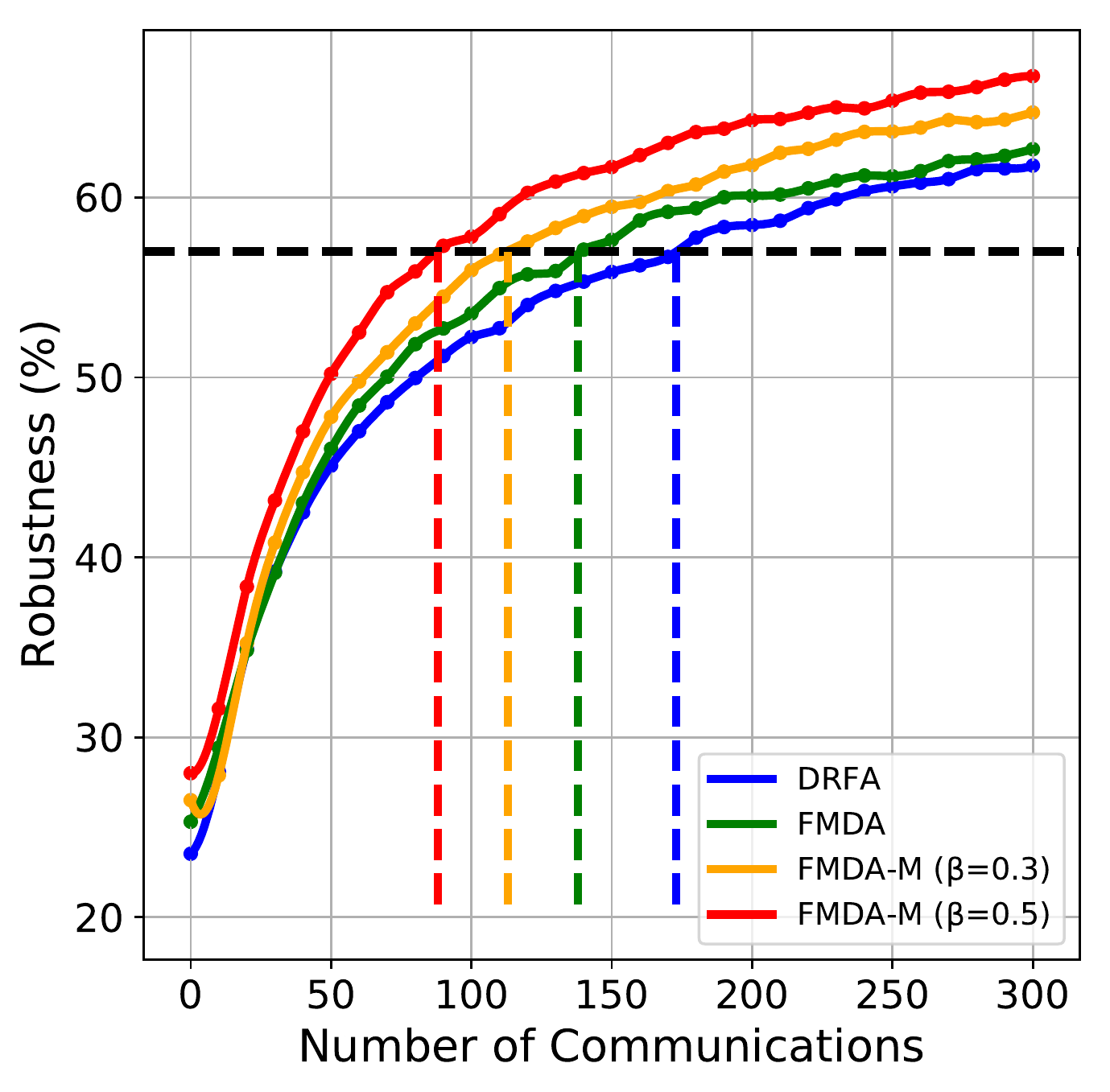}             
\end{minipage}}
\subfigure[Fairness]{
\begin{minipage}[t]{0.48\linewidth}
\includegraphics[width=1.55in]{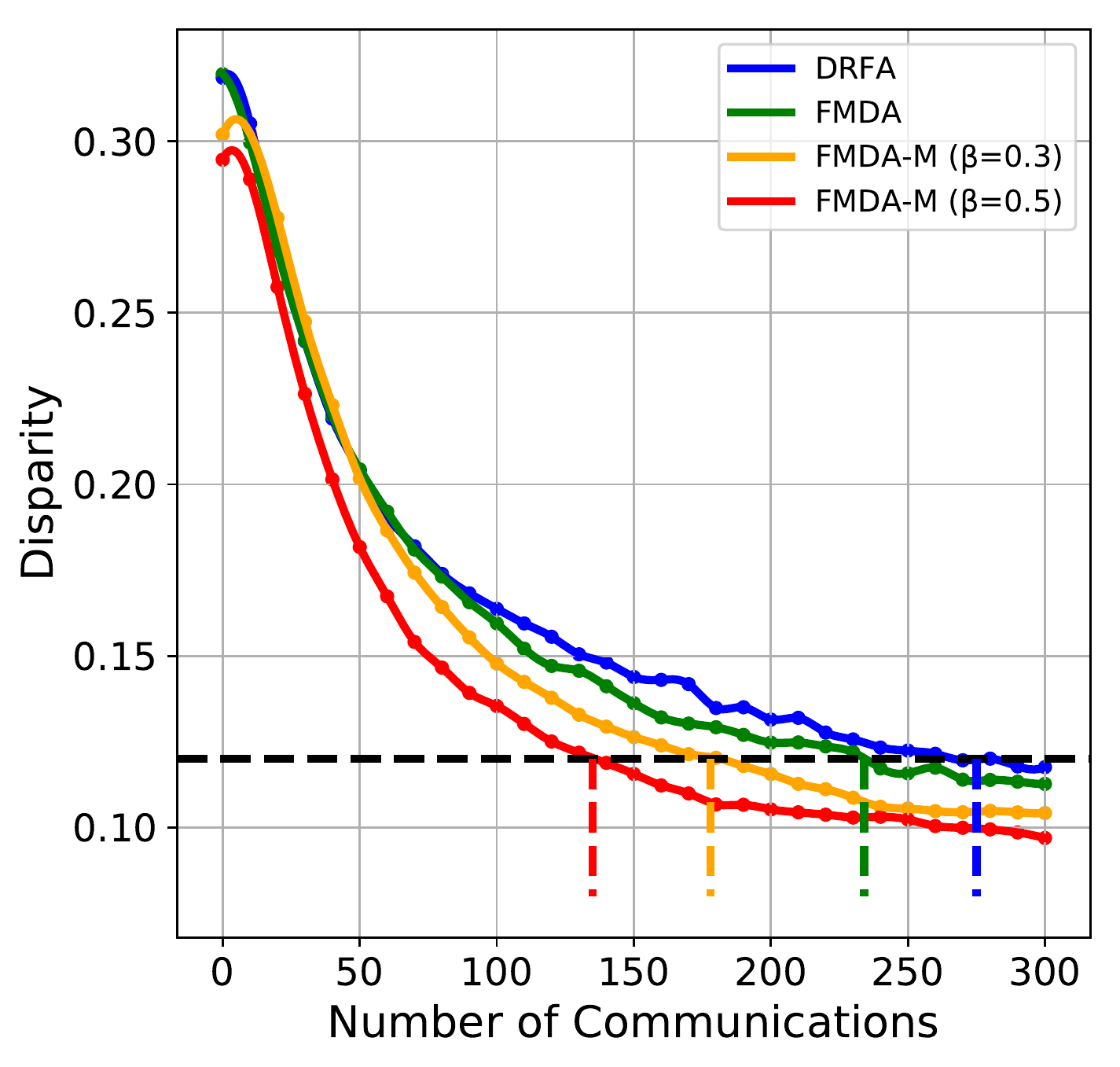}             
\end{minipage}}
\vspace{-1.2em}
\caption{Experimental results of efficiency of different algorithms on Fashion-MNIST in extremely Non-IID setting at attribute level.}
\vspace{-1.35em}
\label{fig-ablation-study}    
\end{figure}

\subsection{Efficiency and Ablation Study}
To evaluate the efficiency of our FMDA-M algorithm and demonstrate how it works, we run the following algorithms on Fashion-MNIST dataset in extremely Non-IID setting: (i) DRFA~\cite{deng2021distributionally}: DRFA algorithm is proposed to solve the min-max problem in federated setting. (ii) FMDA. (iii) FMDA-M ($\beta = \beta_{\theta}$ = $\beta_{\bm{\lambda}} = 0.3$). (iv) FMDA-M ($\beta = \beta_{\theta}$ = $\beta_{\bm{\lambda}} = 0.5$). Note that the above algorithms share the same optimization objective. We report the learning curves of models in terms of attribute-level robustness and fairness over 300 rounds of communications, as shown in Figure~\ref{fig-ablation-study}. Results in other settings are in the appendix.

The results show that the proposed FMDA outperforms DRFA in terms of convergence rate. The most likely reason is that FMDA adopt mirror ascent based on Bregman divergence, instead of projection operation based on Euclidean distance, to update the group weights, which prevents the weights from being too hard to guarantee convergence stability.
We observe that FMDA-M is more efficient and can achieve the same level as others with fewer number of communication rounds, because the momentum term can modify the direction of the current gradients to accelerate convergence (more details in appendix).

\section{Conclusion}
In this paper, we view the accuracy of federated model as the resource to be allocated. We investigate existing fairness problems in FL and propose the goal of unified group fairness, which is to achieve client-level, attribute-level and
agnostic distribution fairness simultaneously. To achieve it, we propose a fair FL framework based on a unified federated uncertainty set with theoretical analysis. We also develop an efficient FMDA-M algorithm with convergence guarantee. Empirically, extensive results show that our FMDA-M outperforms existing fair FL algorithms significantly and shows extremely stable performance in various environments.

\bibliography{aaai22}

\end{document}